\documentclass[twoside,11pt]{article}
\usepackage{arxiv}
\usepackage{hyperref}

\usepackage{amsmath, amsfonts, mathtools}
\usepackage{tabularx,ragged2e,booktabs,multirow} 
\usepackage{xcolor}
\usepackage{enumitem}
\usepackage{cases}
\usepackage{url}            
\usepackage{standard_macros}
\usepackage[textfont=it]{caption}
\usepackage{subcaption}
\usepackage[ruled]{algorithm2e}
\usepackage{todonotes}

\SetKwProg{Fn}{Function}{}{}

\jmlrheading{-}{2018}{-}{4/18}{-/-}{rontsis18}{Nikitas Rontsis and Michael A. Osborne and Paul J. Goulart}

\ShortHeadings{Distributionally Ambiguous Optimization for Batch Bayesian Optimization}{Rontsis, Osborne and Goulart}
\firstpageno{1}

\begin{document}

\title{Distributionally Ambiguous Optimization for\\ Batch Bayesian Optimization}

\author{\name Nikitas Rontsis \email nikitas.rontsis@eng.ox.ac.uk \\
	\name Michael A. Osborne \email mosb@robots.ox.ac.uk \\
	\name Paul J. Goulart \email paul.goulart@eng.ox.ac.uk \\
	\addr Department of Engineering Science\\
	University of Oxford\\
	Oxford, OX1 3PN UK
}


\maketitle

\begin{abstract}
	We propose a novel, theoretically-grounded, acquisition function for Batch Bayesian optimization informed by insights from distributionally ambiguous optimization. Our acquisition function is a lower bound on the well-known Expected Improvement function, which requires evaluation of a Gaussian Expectation over a multivariate piecewise affine function.   Our bound is computed instead by evaluating the best-case expectation over all probability distributions consistent with the same mean and variance as the original Gaussian distribution. Unlike alternative approaches, including Expected Improvement, our proposed acquisition function avoids multi-dimensional integrations entirely, and can be computed exactly -- even on large batch sizes -- as the solution of a tractable convex optimization problem. Our suggested acquisition function can also be optimized efficiently, since first and second derivative information can be calculated inexpensively as by-products of the acquisition function calculation itself. We derive various novel theorems that ground our work theoretically and we demonstrate superior performance via simple motivating examples, benchmark functions and real-world problems.
\end{abstract}


\section{Introduction}
When dealing with numerical optimization problems in engineering applications, one is often faced with the optimization of an expensive process that depends on a number of tuning parameters. Examples include the outcome of a biological experiment, the training of large scale machine learning algorithms or the outcome of exploratory drilling.  We are concerned with problem instances wherein there is the capacity to perform $k$ experiments in \emph{parallel} and, if needed, repeatedly select further batches with cardinality $k$ as part of some sequential decision making process. Given the cost of the process, we wish to select the parameters at each stage of evaluations carefully so as to optimize the process using as few experimental evaluations as possible.

It is common to assume a surrogate model for the outcome $f: \mathbb{R}^n \mapsto \mathbb{R}$ of the process to be optimized. This model, which is based on both prior assumptions and past function evaluations, is used to determine a collection of $k$ input points for the next set of evaluations. Bayesian Optimization provides an elegant surrogate model approach and has been shown to outperform other state-of-the-art algorithms on a number of challenging benchmark functions \citep{Jones2001}. It models the unknown function $f$ with a Gaussian Process (GP) \citep{Rasmussen2005}, a probabilistic function approximator which can incorporate prior knowledge such as smoothness, trends, etc.

A comprehensive introduction to GPs can be found in \cite{Rasmussen2005}. In short, modeling a function with a GP amounts to modeling the function as a realization of a stochastic process. In particular, we assume that the outcomes of function evaluations are normally distributed random variables with known \emph{prior} mean function $m : \mathbb{R}^n \mapsto \mathbb{R}$ and \emph{prior} covariance function $\kappa : \mathbb{R}^n \times \mathbb{R}^{n} \mapsto \mathbb{R}$. Prior knowledge about $f$, such as smoothness and trends, can be incorporated through judicious choice of the covariance function $\kappa$, while the mean function $m$ is commonly assumed to be zero.
A training dataset $\mathcal{D} = (X^d, y^d)$ of $\ell$ past function evaluations $y^d_i = f(X^d_i)$ for $i = 1\dots \ell$, with $y^d \in \mathbb{R}^\ell, X^d \in \mathbb{R}^{\ell \times n}$ is then used to calculate the \emph{posterior} of $f$.

The celebrated GP regression equations \citep{Rasmussen2005} give the posterior
\begin{equation}
	y | \mathcal{D} \sim \mathcal{N}(\mu(X),
	\Sigma(X))
	\label{eqn:GPPrediction}
\end{equation}
on a batch of $k$ test locations $X \in \mathbb{R}^{k \times n}$ as a multi-variate normal distribution in a closed form formula. The posterior mean value $\mu$ and variance $\Sigma$ depend also on the dataset $\mathcal{D}$, but we do not explicitly indicate this dependence in order to simplify the notation.  Likewise, the posterior $y | \mathcal{D}$ is a normally distributed random variable whose mean $\mu(X)$ and covariance $\Sigma(X)$ depend on $X$, but we do not make this explicit.

Given the surrogate model, we wish to identify some selection criterion for choosing the next batch of points to be evaluated. Such a selection criterion is known as an \emph{acquisition function} in the terminology of Bayesian Optimization. Ideally, such an acquisition function would take into account the number of remaining evaluations that we can afford, e.g.\ by computing a solution via dynamic programming to construct an optimal sequence of policies for future batch selections. However, a probabilistic treatment of such a criterion is computationally intractable, involving multiple nested minimization-marginalization steps \citep{Gonzalez2015}.

To avoid this computational complexity, myopic acquisition functions that only consider the one-step return are typically used instead. For example, one could choose to minimize the one-step \emph{Expected Improvement} (described more fully in \S \ref{sec:expected_improvement}) over the best evaluation observed so far, or maximize the probability of having an improvement in the next batch over the best evaluation. Other criteria use ideas from the bandit \citep{Desautels2014} and information theory \citep{Shah2015} literature. In other words, the intractability of the multistep lookahead problem has spurred instead the introduction of a wide variety of myopic heuristics for batch selection.

\subsection{Expected improvement} 
\label{sec:expected_improvement}
We will focus on the (one-step) Expected Improvement criterion, which is a standard choice and has been shown to achieve good results in practice \citep{Snoek2012}. In order to give a formal description we first require some definitions related to the optimization procedure of the original process. At each step of the optimization procedure, define $y^d \in \mathbb{R}^\ell$ as the vector of $\ell$ past function values evaluated at the points $X^d \in \mathbb{R}^{\ell \times n}$, and $X \in \mathbb{R}^{k \times n}$ as a candidate set of $k$ points for the next batch of evaluations. Then the classical expected improvement acquisition function is defined as
\begin{equation}
	\begin{gathered}
		\alpha(X) = \mathbb{E}[\min(y_{1}, \dots, y_{k}, \underline{y^d}) |
		\mathcal{D}] - \underline{y^d} \\
		\quad \text{with} \; y | \mathcal{D} \sim \mathcal{N}\bigl(\mu(X),
		\Sigma(X)\bigr)
		\label{eqn:ExpectedImprovement}
	\end{gathered}
\end{equation}
where $\underline{y^d}$ is the element-wise minimum of $y^d$, i.e.\ the minimum value of the function $f$ achieved so far by any known function input. In the above definition we assume perfect knowledge of $\underline{y^d}$, which implies a noiseless objective. A noisy objective requires the introduction of heuristics discussed in detail in \cite{Picheny2013}. For the purposes of clarity, a noiseless objective is assumed for the rest of the document. This is not constraining, as most of the heuristics discussed in \citep{Picheny2013} are compatible with the theoretical analysis presented in the rest of the paper.

Selection of a batch of points to be evaluated with optimal expected improvement amounts to finding some $X \in \arg\min\left[\alpha(X)\right].$ Unfortunately, direct evaluation of the acquisition function $\alpha$ requires the $k$--dimensional integration of a piecewise affine function; this is potentially a computationally expensive operation. This is particularly problematic for gradient-based optimization methods, wherein $\alpha(X)$ may be evaluated many times when searching for a minimizer. Regardless of the optimization method used, such a minimizer must also be computed again for every step in the original optimization process, i.e.\ every time a new batch of points is selected for evaluation. Therefore a tractable acquisition function should be used. In contrast to \eqref{eqn:ExpectedImprovement}, the acquisition function we will introduce in Section \ref{sec:main_result} avoids expensive integrations and can be calculated efficiently with standard software tools.

\subsection{Related work} 

Despite the intractability of \eqref{eqn:ExpectedImprovement}, \citet{Chevalier2013} presented an efficient way of approximating $\alpha$ and its derivative $d\alpha/dX$ \citep{Marmin2015} by decomposing it into a sum of $q$--dimensional Gaussian Cumulative Distributions, which can be calculated efficiently using the seminal work of \citet{Genz2009}. There are two issues with this approach: First the number of calls to the $q$--dimensional Gaussian Cumulative Distribution grows quadratically with respect to the batch size $q$, and secondly, there are no guarantees about the accuracy of the approximation or its gradient. Indeed, approximations of the multi-point expected improvement, as calculated with the \texttt{R} package \texttt{DiceOptim} \citep{Roustant2012} can exhibit points of failure in simple examples (see Figure \ref{fig:qei_problem} and Table \ref{table:Timings}). To avoid these issues, \citet{Gonzalez2016} and \citet{Ginsbourger2009} rely on heuristics to derive a multi-point criterion. Both methods choose the batch points in a greedy, sequential way, which restricts them from exploiting the interactions between the batch points in a probabilistic manner. Finally Parallel Predictive Entropy Search (PPES), a non-greedy information theoretic acquisition function was derived by \citet{Shah2015}.

\section{Distributionally ambiguous optimization for Bayesian optimization}
\label{sec:main_result}
We now proceed to the main contribution of the paper. We draw upon ideas from the Distributionally Ambiguous Optimization community to derive a novel, tractable, acquisition function that lower bounds the expectation in~\eqref{eqn:ExpectedImprovement}. Our acquisition function:
\begin{itemize}
	\itemsep0.1em
	\item is theoretically grounded;
	\item is numerically reliable and consistent, unlike Expected Improvement-based alternatives (see \S\ref{sec:empirical_analysis});
	\item is fast and scales well with the batch size; and
	\item provides first and second order derivative information inexpensively.
\end{itemize}
In particular, we use the GP posterior \eqref{eqn:GPPrediction} derived from the GP to determine the mean $\mu(X)$ and variance $\Sigma(X)$ of $y | \mathcal{D}$ given a candidate batch selection $X$, but we thereafter ignore the Gaussian assumption and consider only that $y | D$ has a distribution embedded within a family of distributions $\mathcal{P}$ that share the mean $\mu(X)$ and covariance $\Sigma(X)$ calculated by the standard GP regression equations. In other words, we define
\[
	\mathcal{P}(\mu,\Sigma) = \set{\mathbb P}{\mathbb{E}_\mathbb{P}[\xi] = \mu,
		\mathbb{E}_\mathbb{P}[\xi\xi^T] = \Sigma}.
\]
We denote the set $\mathcal{P}(\mu(X),\Sigma(X))$ simply as $\mathcal{P}$ where the context is clear.   Note in particular that $\mathcal{N}(\mu,\Sigma) \in \mathcal{P}(\mu,\Sigma)$ for any choice of mean $\mu$ or covariance $\Sigma$.

One can then construct lower and upper bound for the Expected Improvement by minimizing or maximizing over the set $\mathcal{P}$ respectively, i.e.\ by writing
\begin{equation}
	\inf_{\mathbb{P} \in \mathcal{P}} \mathbb{E}_{\mathbb{P}}[g(\xi)]
	\leq \mathbb{E}_{\mathcal{N}(\mu, \Sigma)}[g(\xi)]
	\leq \sup_{\mathbb{P} \in \mathcal{P}} \mathbb{E}_{\mathbb{P}}[g(\xi)]
	\label{eqn:ExpectedImprovementBounds}
\end{equation}
where the random vector $\xi \in \mathbb{R}^k$ and the function $g : \mathbb{R}^k \mapsto \mathbb{R}$ are chosen such that $\alpha(X) = \mathbb{E}_{\mathcal{N}(\mu, \Sigma)}[g(\xi)]$, i.e. $\xi = y|\mathcal{D}$ and
\begin{equation}\label{eqn:definition_g}
	g(\xi) = \min(\xi_1, \dots, \xi_k, \underline{y^d}) - \underline{y^d}.
\end{equation}
Observe that the middle term in \eqref{eqn:ExpectedImprovementBounds} is equivalent to the expectation in~\eqref{eqn:ExpectedImprovement}.

Perhaps surprisingly, both of the bounds in~\eqref{eqn:ExpectedImprovementBounds} are computationally tractable even though they seemingly require optimization over the infinite-dimensional (but convex) set of distributions $\mathcal{P}$. For either case, these bounds can be computed exactly via transformation of the problem to a tractable, convex semidefinite optimization problem using distributionally ambiguous optimization techniques \citep{Zymler2013}.

We will focus on the lower bound $\inf_{\mathbb{P} \in \mathcal{P}} \mathbb{E}_{\mathbb{P}}[g(\xi)]$ in \eqref{eqn:ExpectedImprovementBounds}, hence adopting an \emph{optimistic} modeling approach. This is because the upper bound is of limited use as it can be shown to be trivial to evaluate and independent of $\Sigma$. Hence, we informally assume that the distribution of function values is such that it yields the lowest possible expectation compatible with the mean and covariance computed by the GP, which we put together in the second order moment matrix $\Omega$ of the posterior as
\begin{equation}
	\Omega \eqdef
	\begin{bmatrix} 
		\Sigma + \mu\mu^T & \mu \\
		\mu^T             & 1
	\end{bmatrix} 
	\label{eqn:Omega}.
\end{equation}
We will occasionally write this explicitly as $\Omega(X)$ to highlight the dependency of the second order moment matrix on $X$.

The following result says that the lower (i.e.\ optimistic) bound in \eqref{eqn:ExpectedImprovementBounds} can be computed via the solution of a convex optimization problem whose objective function is linear in $\Omega$:

\begin{theorem}\label{thm:ValueTheorem}
	For any $\Sigma \succ 0$ the optimal value of the semi-infinite optimization problem
	\[
		\inf_{\mathbb{P} \in \mathcal{P}(\mu, \Sigma)} \mathbb{E}_{\mathbb{P}}[g(\xi)]
	\]
	coincides with the optimal value of the following semidefinite program:
	\begin{equation}
		\begin{aligned}
			p(\Omega) \eqdef \; &\sup
			& & \langle \Omega, M \rangle\\
			  & \text{s.t.} &   & M - C_i \preceq 0, \quad i = 0,\dots,k,
		\end{aligned}
		\label{eqn:PrimalOptProb}
		\tag{$P$}
	\end{equation}
	where $M \in \mathbb{S}^{k+1}$ is the decision variable, and $C_0 \eqdef 0,$
	\begin{equation}
		C_i \eqdef
		\left[
			\begin{array}{cc}
				0       & e_i/2            \\
				e_i^T/2 & -\underline{y^d}
			\end{array}
		\right],
		\qquad i = 1,\dots,k,
		\qquad
		\label{eqn:PrimalConstraints}
	\end{equation}
	are auxiliary matrices defined using 
	$\underline{y^d}$ and the standard basis vectors $e_i$ in $\mathbb{R}^k$.
\end{theorem}

\begin{proof}
	See Appendix \ref{app:PrimalOptProb}.
\end{proof}

Problem \eqref{eqn:PrimalOptProb} is a semidefinite program (SDP). SDPs are convex optimization problems with a linear objective and convex conic constraints (i.e. constraints over the set of symmetric matrices $\mathbb{S}^{k}$, positive semidefinite/definite matrices $\mathbb{S}^{k}_{+}$/$\mathbb{S}^{k}_{++}$) and can be solved to global optimality with standard software tools~\citep{Sturm1999, SCS2016}. We therefore propose the computationally tractable acquisition function
\begin{equation*}
	\bar \alpha(X) \eqdef p\bigl(\Omega(X)\bigr) \leq \alpha (X) \quad \forall X
	\in
	\mathbb{R}^{k\times n},
\end{equation*}
which we will call \emph{Optimistic Expected Improvement (OEI)}, as it is an optimistic variant of the Expected Improvement function in \eqref{eqn:ExpectedImprovement}.

This computational tractability comes at the cost of inexactness in the bounds \eqref{eqn:ExpectedImprovementBounds}, which is a consequence of minimizing over a set of distributions containing the Gaussian distribution as just one of its members. Indeed, it can be proved that the minimizing distribution is discrete with $k+1$ possible outcomes and can be constructed by the Lagrange multipliers of \eqref{eqn:PrimalOptProb} \citep{VanParysThesis2015}. We show in \S\ref{sec:empirical_analysis} that this inexactness is of limited consequence in practice and it mainly renders the acquisition function more explorative. Nevertheless, there remains significant scope for tightening the bounds in  \eqref{eqn:ExpectedImprovementBounds} via imposition of additional convex constraints on the set $\mathcal{P}$, e.g.\ by restricting $\mathcal{P}$ to symmetric or unimodal distributions \citep{VanParys2015}.   Most of the results in this work would still apply, \emph{mutatis mutandis}, if such structural constraints were to be included.

In contrast to the side-effect of inexactness, the distributional ambiguity is useful for integrating out the uncertainty of the GP's hyperparameters efficiently for our acquisition function. Given $q$ samples of the hyperparameters, resulting in $q$ second order moment matrices $\{\Omega_i\}_{i=1,\dots,q}$, we can estimate the resulting second moment matrix $\tilde \Omega$ of the marginalized, non-Gaussian, posterior as
\begin{equation*}
	\tilde \Omega \approx \frac{1}{q}\sum_{i=1}^q \Omega_i.
\end{equation*}
Due to the distributional ambiguity of our approach, both bounds of \eqref{eqn:ExpectedImprovementBounds} can be calculated directly based on $\tilde \Omega$, hence avoiding multiple calls to the acquisition function.

Although the value of $p(\Omega)$ for any fixed $\Omega$ is computable via solution of an SDP, the non-convexity of the GP posterior \eqref{eqn:GPPrediction} that defines the mapping $X \mapsto \Omega(X)$ means that $\bar \alpha (X) = p \bigl(\Omega(X) \bigr)$ is still non-convex in $X$. This is unfortunate, since we ultimately wish to minimize $\bar \alpha(X)$ in order to identify the next batch of points to be evaluated experimentally.

However we can still optimize $\bar \alpha$ locally via non-linear programming. We will establish that a second order method is applicable by showing that $\bar \alpha (X)$ is twice differentiable under mild conditions. Such an approach would also be efficient as the hessian of $\bar \alpha$ can be calculated inexpensively. To show these results we will begin by considering the differentiability of $p$ as a function of $\Omega$.
\begin{theorem}\label{thm:gradientTheorem}
	The optimal value function $p : \mathbb{S}^{k+1}_{++} \mapsto \mathbb{R}$ defined in problem \eqref{eqn:PrimalOptProb} is differentiable on
	its domain with $\partial p(\Omega)/ \partial \Omega = \bar M(\Omega)$,
	where $\bar M(\Omega)$ is the unique optimal solution of
	\eqref{eqn:PrimalOptProb} at $\Omega$.
\end{theorem}
\begin{proof}
	See Appendix \ref{app:DerivativeCalculations}.
\end{proof}
The preceding result shows that $\partial p(\Omega)/ \partial \Omega$ is produced as a byproduct of evaluation of $\inf_{\mathbb{P} \in \mathcal{P}} \mathbb{E}_{\mathbb{P}}[g(\xi)]$, since it is simply $\bar M(\Omega)$, the unique optimizer of \eqref{eqn:PrimalOptProb}. The simplicity of this result suggests consideration of second derivative information of $p(\Omega)$, i.e.\ derivatives of $\bar M(\Omega)$. The following result proves that this is well defined and tractable for any $\Omega \succ 0$:
\begin{theorem}\label{thm:dotM}
	The optimal solution function $\bar M: \mathbb{S}_{++}^{k+1} \mapsto \mathbb{S}^{k+1}$ in problem \eqref{eqn:PrimalOptProb} is differentiable on $\mathbb{S}_{++}^{k+1}$. Every directional derivative of $\bar M (\Omega)$ is the unique solution to a sparse linear system with $\mathcal{O}(k^3)$ nonzeros.
\end{theorem}
\begin{proof}
	See Appendix \ref{app:Mdot}.
\end{proof}
We can now consider the differentiability of $\bar \alpha = p\circ\Omega$. The following Corollary establishes this under certain conditions.
\begin{corollary}
	$\bar \alpha: \mathbb{R}^{k \times n} \mapsto \mathbb{R}$ is twice differentiable on any $X$ for which $\Sigma(X) \succ 0$ and the mean and kernel functions of the underlying GP are twice differentiable.
\end{corollary}
\begin{proof}
	By examining the GP Regression equations \citep{Rasmussen2005} and Equation \eqref{eqn:Omega}, we conclude that $\Omega(X)$ is twice differentiable on $\mathbb{R}^{k\times n}$ if the kernel and mean functions of the underlying Gaussian Process are twice differentiable. Hence, $\bar \alpha (X) = p(\Omega(X))$ is twice differentiable for any $\Omega(X) \succ 0$ as a composition of twice differentiable functions. Examining \eqref{eqn:Omega} reveals that $\Omega(X) \succ 0$ is equivalent to $\Sigma(X) \succ 0$, which concludes the proof.
\end{proof}
A rank deficient $\Sigma(X) \nsucc 0$ implies perfectly correlated outcomes. At these points both OEI and QEI are, in general, non-differentiable as shown in Proposition \ref{prop:nonDifferentiability}. This contradicts \cite{Marmin2015} which claims differentiability of QEI in the general setting. However, this is not constraining in practice as both QEI and OEI can be calculated by considering a smaller, equivalent problem. It is also not an issue for descent based methods for minimizing $\bar \alpha$, as a descent direction can be obtained by an appropriate perturbation of the perfectly correlated points.

We are now in a position to derive expressions for the gradient and the hessian of $\bar \alpha = p\circ \Omega$. For simplicity of notation we consider derivatives over $\bar x = \text{vec}(X)$. Application of the chain rule to $\bar \alpha(\bar x) = p(\Omega(\bar x))$ gives:
\begin{equation}
	\frac{\partial \bar \alpha (\bar x)}{\partial \bar x_{(i)}}
	=
	\innerprod{
		\frac{\partial p(\Omega)}
		{\partial \Omega}
		}{
		\frac{\partial \Omega(\bar x )}
		{\partial \bar x_{(i)}}
		} =
	\innerprod{
		\bar M(\Omega)
		}{
		\frac{\partial \Omega(\bar x)}
		{\partial \bar x_{(i)}}
	}
	\label{eqn:gradient}
\end{equation}
Note that the second term in the rightmost inner product above depends on the particular choice of covariance function $\kappa$ and mean function $m$. It is straightforward to compute \eqref{eqn:gradient} in modern graph-based \texttt{autodiff} frameworks, such as the \texttt{TensorFlow}-based \texttt{GPflow}. Differentiating again \eqref{eqn:gradient} gives the hessian of $\bar \alpha$:
\begin{equation}
	\frac{\partial^2 \bar \alpha(\bar x)}{\partial \bar x_{(i)} \partial \bar x_{(j)}}
	=
	\frac{\partial}{\partial x_{(i)}}
	\innerprod{
		\bar M(\Omega)
		}{
		\frac{\partial \Omega(\bar x)}
		{\partial \bar x_{(j)}}
	}
	=
	\innerprod{
		\bar M(\Omega)
	}
	{
		\frac{\partial^2
			\Omega(\bar x)}
		{\partial \bar x_{(i)} \partial \bar x_{(j)}}
	}
	+
	\innerprod{
		\frac{\partial \bar M(\Omega(\bar x))}{\partial \bar x_{(i)}}
		}{
		\frac{\partial \Omega(\bar x)}{\partial \bar x_{(j)}}
	}
	\label{eqn:ChainRule}
\end{equation}
where $\partial \bar M / \partial \bar x_{(i)}$ is the directional derivative of $\bar M(\Omega)$ across the perturbation $\partial \Omega(\bar x ) / \partial \bar x_{(i)}$. According to Theorem \ref{thm:dotM}, each of these directional derivatives exists and can be computed via solution of a sparse linear system.

\section{Empirical analysis} \label{sec:empirical_analysis}
\begin{table}[ht]
	\caption{List of acquisition functions} \label{sample-table}
	\begin{center}
		\begin{tabular}{lc}
			\toprule
			{Key} & {Description} \\
			\hline \\
			OEI  & 
			Optimistic Expected Improvement
			\emph{(Our novel algorithm)} \\
		QEI &
			Multi-point Expected Improvement
			\citep{Marmin2015} \\
		QEI-CL &
			Constant Liar (\emph{``mix''} strategy)
			\citep{Ginsbourger2010} \\
		LP-EI &
			Local Penalization Expected Improvement
			\citep{Gonzalez2016} \\
		BLCB &
			Batch Lower Confidence Bound
			\citep{Desautels2014} \\
		\end{tabular}
	\end{center}
	\label{table:AlgList}
\end{table}
%
In this section we demonstrate the effectiveness of our acquisition function against a number of state-of-the-art alternatives. The acquisition functions we consider are listed in Table \ref{table:AlgList}. We do not compare against PPES as it is substantially more expensive and elaborate than our approach and there is no publicly available implementation of this method.

We show that our acquisition function OEI achieves better performance than alternatives and highlight simple failure cases exhibited by competing methods. In making the following comparisons, extra care should be taken in the setup used. This is because Bayesian Optimization is a multifaceted procedure that depends on a collection of disparate elements (e.g. kernel/mean function choice, normalization of data, acquisition function, optimization of the acquisition function) each of which can have a considerable effect on the resulting performance \citep{Snoek2012,Shahriari2016}. For this reason we test the different algorithms on a unified testing framework, based on \texttt{GPflow}, available online at \href{https://github.com/oxfordcontrol/Bayesian-Optimization}{https://github.com/oxfordcontrol/Bayesian-Optimization}.

Our acquisition function is evaluated via solution of a semidefinite program, and as such it benefits from the huge advances of the convex optimization field. A variety of standard tools exist for solving such problems, including \texttt{MOSEK} \citep{MosekPy}, \texttt{SCS} \citep{ODonoghue2016} and \texttt{CDCS} \citep{Zheng2017}. We chose the first-order \citep{Boyd2004}, freely-available solver \texttt{SCS}, which scales well with batch size and allows for solver warm-starting between acquisition function evaluations.

Warm starting allows for a significant speedup since the acquisition function is evaluated repeatedly at nearby points by the non-linear solver. This results in solving \eqref{eqn:PrimalOptProb} repeatedly for similar $\Omega$. Warm-starting the SDP solver with the previous solution reduces \texttt{SCS}'s iterations by $77\%$ when performing the experiments of Figure $\ref{fig:SyntheticBatch20}$. Moreover, Theorem \ref{thm:gradientTheorem} provides the means for a first-order warm starting. Indeed, the derivative of the solution across the change of the cost matrix $\Omega$ can be calculated, allowing us to take a step in the direction of the gradient and warm start from that point. This reduces \texttt{SCS}'s iterations by a further $43\%$.

Indicative timing results for the calculation of OEI, QEI and their derivatives are listed in Table \ref{table:Timings}. The dominant operation for calculating OEI and its gradient is solving \eqref{eqn:PrimalOptProb}. 
This makes OEI much faster than QEI, which is in line with the complexity of the dominant operation in SDP solvers based on first-order operator splitting methods such as \texttt{SCS} or \texttt{CDCS} which, for our problem,  is $\mathcal{O}\bigl(k^4 \bigr)$. Assume that, given the solution of \eqref{eqn:PrimalOptProb}, we want to also calculate the hessian of OEI. This would entail the following two operations:
\begin{description}
	\item[Calculating $\partial \bar M / \partial X_{(i, j)}$ given $\partial \Omega(X) / \partial X_{(i, j)}.$] According to Lemma \ref{thm:dotM} this can be obtained as a solution to a sparse linear system. We used \texttt{Intel$^{\tiny{\textregistered}}$ MKL PARDISO} to solve efficiently these linear systems.
	\item[Calculate $\partial \Omega(X) / \partial X_{(i, j)}$ and apply chain rules] of \eqref{eqn:ChainRule} to get the hessian of the acquisition function $\bar \alpha = p \circ \Omega$ given the gradient and hessian of $p$. We used \texttt{Tensorflow}'s Automatic Differentiation for this part, without any effort to optimize its performance. Considerable speedups can be brought by e.g. running this part on GPU, or automatically generating low-level code optimized specifically for this operation \citep{Vasilache2018}.
\end{description}
Note that the computational tractability of the hessian is only allowed due to the novelty of Theorem \ref{thm:gradientTheorem} which exploits the ``structure'' of \eqref{eqn:PrimalOptProb}'s optimizer.
\begin{table}[t]
	\caption{Average execution time of the acquisition function, its gradient and hessian when running BO in the Eggholder function on an Intel$^{\tiny{\textregistered}}$ E5-2640v3 CPU. For batch size 40, QEI fails, i.e. it always returns 0 without any warning message. For the execution time of the hessian we assume knowledge of the solution of \eqref{eqn:PrimalOptProb}. Its timing is split into two parts as described in the main text. }
	\vskip 0.15in
	\begin{center}
		\begin{tabular}{cc|c|cc}
			\toprule
			Batch Size & QEI: $\alpha(X), \nabla \alpha(X)$ & OEI: $\bar \alpha(X), \nabla \bar \alpha(X)$ &  \multicolumn{2}{c}{$\nabla^2 \bar \alpha(X)$} \\
			\midrule
			& & Solve \eqref{eqn:PrimalOptProb} & $\partial \bar M$ & \texttt{Tensorflow} Part\\
			2 & $5.6 \cdot 10^{-3}$ & $2.1 \cdot 10^{-4}$ & $5.3 \cdot 10^{-4}$ & $1.0 \cdot 10^{-2}$ \\
			3 & $1.2 \cdot 10^{-2}$ & $3.8 \cdot 10^{-4}$ & $7.5 \cdot 10^{-4}$ & $1.4 \cdot 10^{-2}$ \\ 
			6 & $1.1 \cdot 10^{-1}$ & $1.5 \cdot 10^{-3}$ & $1.7 \cdot 10^{-3}$ & $2.0 \cdot 10^{-2}$ \\
			10 & $1.1$  & $8.2 \cdot 10^{-3}$ & $5.9 \cdot 10^{-3}$ & $3.2 \cdot 10^{-2}$\\
			20 & $2.1 \cdot 10^{1}$ & $4.7 \cdot 10^{-2}$ & $2.3 \cdot 10^{-2}$ & $8.7 \cdot 10^{-2}$ \\
			40 & $-$ & $3.4 \cdot 10^{-1}$ & $1.4 \cdot 10^{-1}$ & $3.5 \cdot 10^{-1}$ \\
			\bottomrule
		\end{tabular}
	\end{center}
	\vskip -0.1in
	\label{table:Timings}
\end{table}

We chose the \texttt{KNITRO v10.3} \citep{Byrd2006} Sequential Quadratic Optimization (SQP) non-linear solver with the default parameters for the optimization of OEI. Explicitly providing the hessian on the experiments of Figure \ref{fig:SyntheticBatch20} reduces \texttt{KNITRO}'s iterations by $49\%$ as compared to estimating the hessian via the symmetric-rank-one (SR1) update method included in the \texttt{KNITRO} suite. Given the inexpensiveness of calculating the hessian and the fact that \texttt{KNITRO} requests the calculation of the hessian less than a third as often as it requires the objective function evaluation we conclude that including the hessian is beneficiary.

We will now present simulation results to demonstrate the performance of OEI in various scenarios.

\subsection{Perfect modeling assumptions}
We first demonstrate that the competing Expected Improvement based algorithms produce clearly suboptimal choices in a simple \mbox{1--dimensional} example. We consider a 1--d Gaussian Process on the interval $[-1, 1]$ with a squared exponential kernel \citep{Rasmussen2005} of lengthscale $1/10$, variance $10$, noise $10^{-6}$ and a mean function $m(x) = (5x)^2$. An example posterior of 10 observations is depicted in Figure \ref{fig:GP_draw}. Given the GP and the 10 observations, we depict the optimal 3-point batch chosen by minimizing each acquisition function. Note that in this case we assume perfect modeling assumptions -- the GP is completely known and representative of the actual process. We can observe in Figure~\ref{fig:GP_draw} that the OEI choice is very close to the one proposed by QEI while being slightly more explorative, as OEI allows for the possibility of more exotic distributions than the Gaussian.
\begin{figure}[t]
	\centering
	\begin{subfigure}[t]{.45\textwidth}
		\includegraphics[width=1\textwidth]{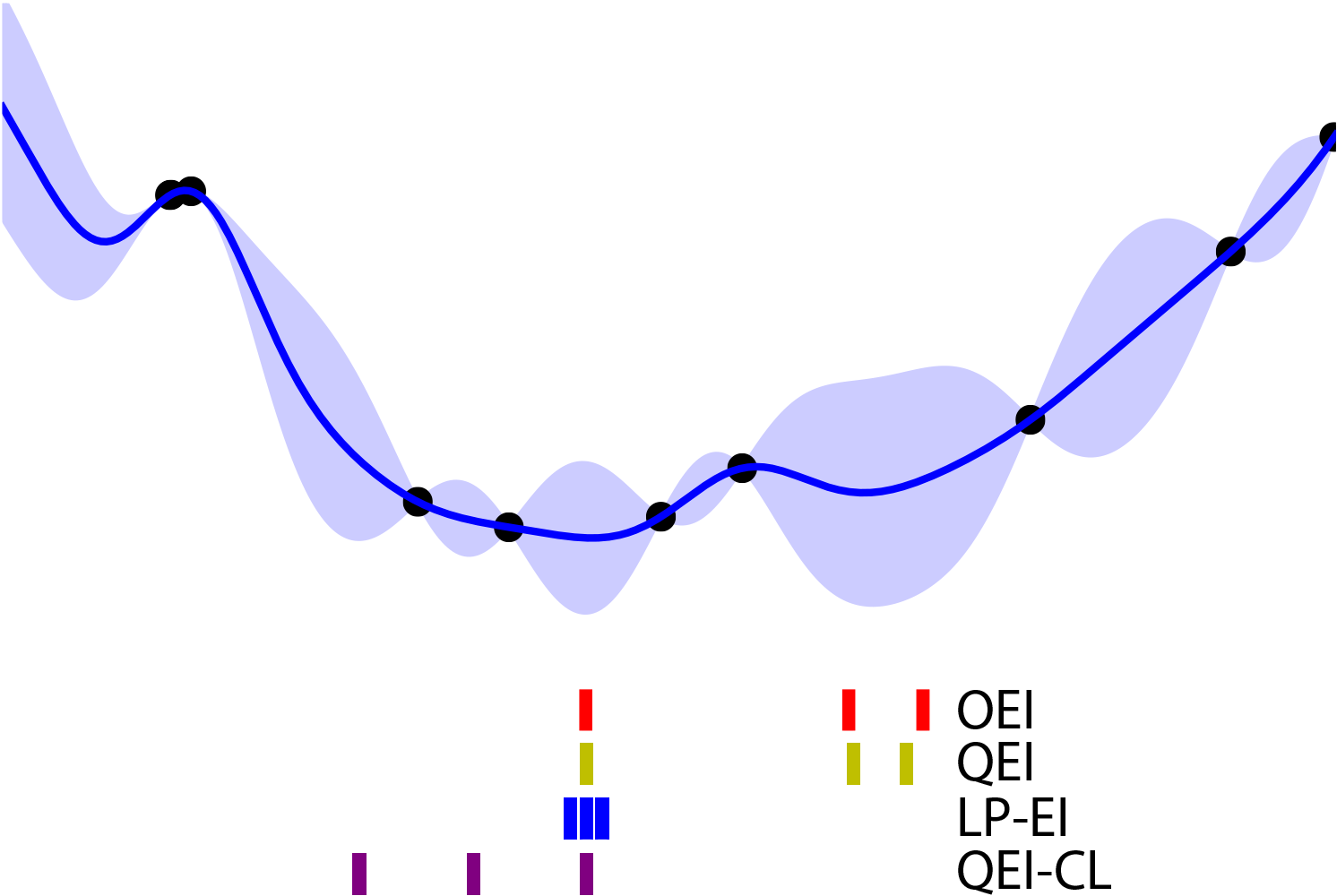}
		\caption{Suggested 3-point batches of different algorithms for a GP posterior given 10 observations. The thick blue line depicts the GP mean, the light blue shade the uncertainty intervals ($\pm$ sd) and the black dots the observations. The locations of the batch chosen by each algorithm are depicted with colored vertical lines  at the bottom of the figure.}
		\label{fig:GP_draw}
	\end{subfigure}
	\quad
	\begin{subfigure}[t]{.45\textwidth}
		\centering
		\includegraphics[width=.9\textwidth]{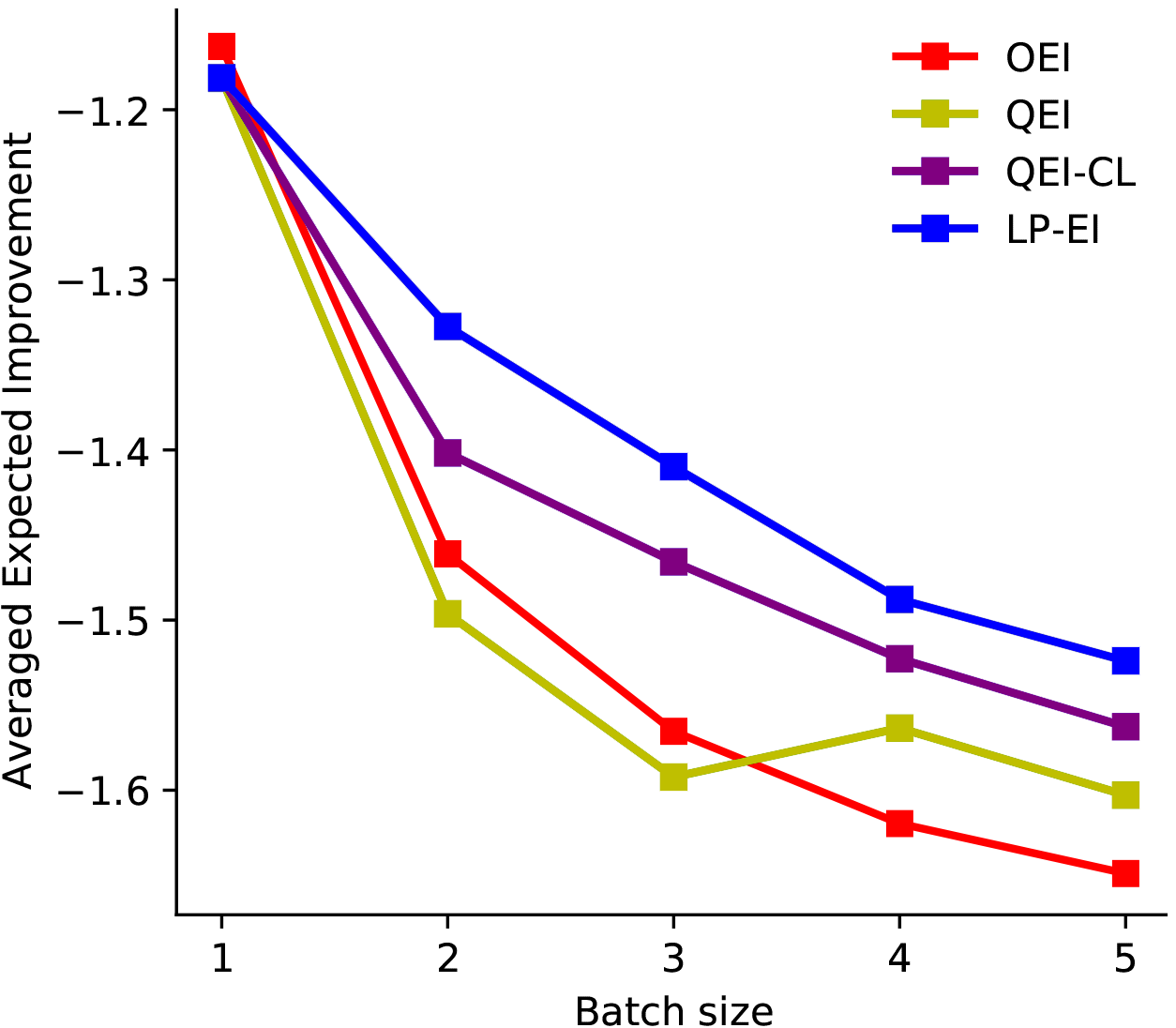}
		\caption{Averaged one-step Expected Improvement on 200 GP posteriors of sets of 10 observations with the same generative characteristics (kernel, mean, noise) as the one in Figure \ref{fig:GP_draw} for different algorithms across varying batch size.}
		\label{fig:EI_vs_batch_size}
	\end{subfigure}
\end{figure}

In contrast the LP-EI heuristic samples almost at the same point all three times. This can be explained as follows: LP-EI is based on a Lipschitz criterion to penalize areas around previous choices. However, the Lipschitz constant for this function is dominated by the end points of the function (due to the quadratic trend), which is clearly not suitable for the area of the minimizer (around zero), where the Lipschitz constant is approximately zero. On the other hand, QEI-CL favors suboptimal regions. This is because QEI-CL postulates outputs equal to the mean value of the observations which significantly alter the GP posterior.

Testing the algorithms on 200 different posteriors, generated by creating sets of 10 observations drawn from the previously defined GP, suggests OEI as the clear winner.  For each of the 200 posteriors, each algorithm chooses a batch, the performance of which is evaluated by sampling the multipoint expected improvement \eqref{eqn:ExpectedImprovement}. The averaged results are depicted in Figure \ref{fig:EI_vs_batch_size}. For a batch size of 1 all of the algorithms perform the same, except for OEI which performs slightly worse due to the relaxation of Gaussianity. For batch sizes 2-3, QEI is the best strategy, while OEI is a very close second. For batch sizes 4-5 OEI performs significantly better. The deterioration of the performance for QEI in batch sizes 4 and 5 can be explained in Figure \ref{fig:qei_problem}. In particular, after batch size 3, the calculation of QEI via the \texttt{R} package \texttt{DiceOptim} exhibits some points of failure, leading to suboptimal choices. Figure \ref{fig:qei_problem} also explains the very good performance of OEI. Although always different from the sampled estimate, it is \emph{reliable} and closely approximates the actual expected improvement in the sense that their optimizers and level sets are in close agreement.
\begin{figure}[t]
	\centering
	\includegraphics[width=.5\textwidth]{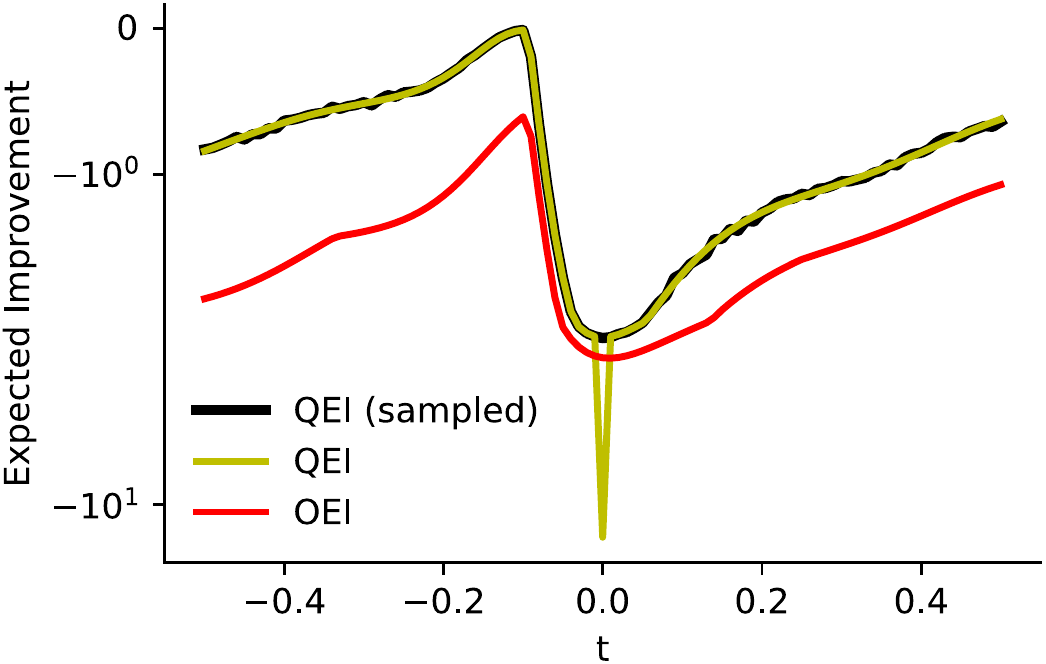}
	\caption{Demonstration of a case of failure in the calculation of the Multi-Point Expected Improvement with the \texttt{R} package \texttt{\tt DiceOptim v2} \citep{Roustant2012}. The figure depicts how the expected improvement varies around a starting batch of 5 points, $X_0$, when perturbing the selection across a specific direction $X_t$ i.e. $f(t) = \mathrm{EI}(X_0 + tX_t)$. The \texttt{\tt DiceOptim}'s calculation is almost always very close to the sampled estimate, except for a single breaking point. In contrast our acquisition function OEI, although always different from the sampled estimate, is \emph{reliable} and closely approximates QEI in the sense that their optimizers and level sets nearly coincide.}
	\label{fig:qei_problem}
\end{figure}

\subsection{Synthetic functions}
\begin{figure}[!ht]
	\centering
	\begin{subfigure}{.49\textwidth}
		\centering
		\includegraphics[width=1\textwidth]{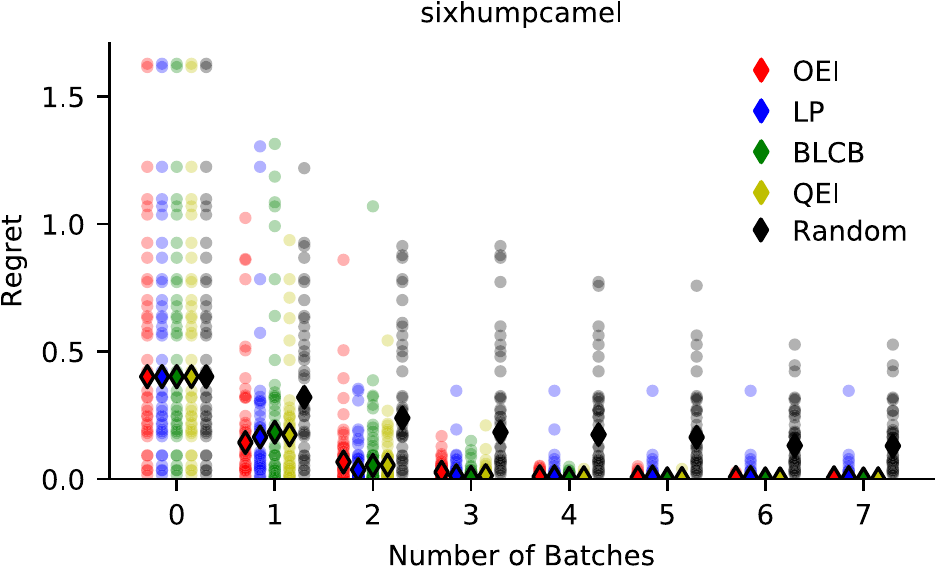}
	\end{subfigure}
	\begin{subfigure}{.49\textwidth}
		\centering
		\includegraphics[width=1\textwidth]{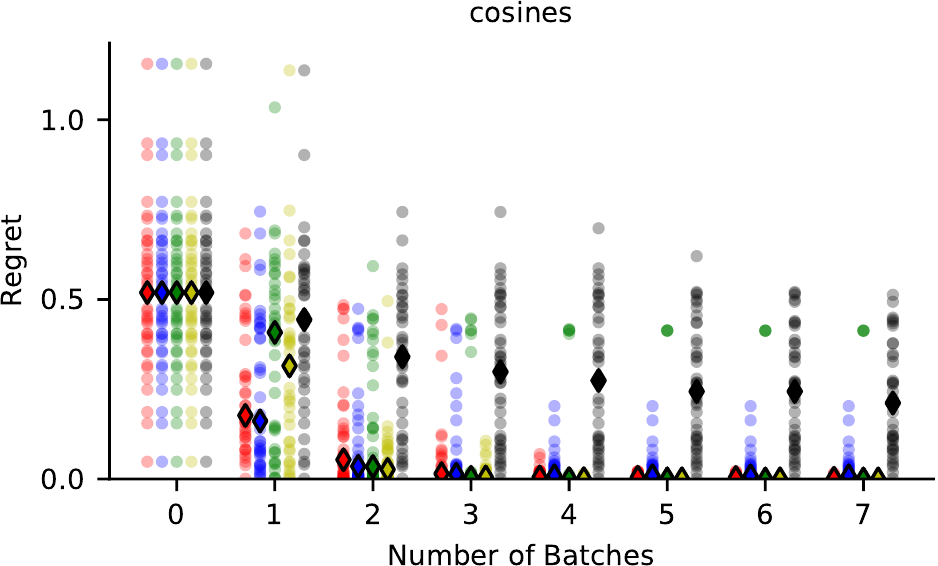}
	\end{subfigure}
	\begin{subfigure}{.49\textwidth}
		\centering
		\includegraphics[width=1\textwidth]{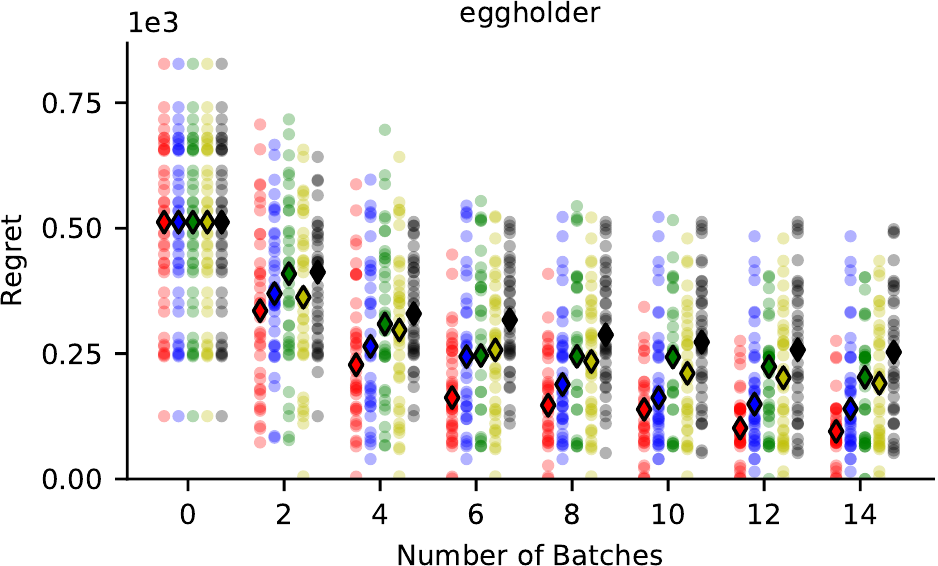}
	\end{subfigure}
	\begin{subfigure}{.49\textwidth}
		\centering
		\includegraphics[width=1\textwidth]{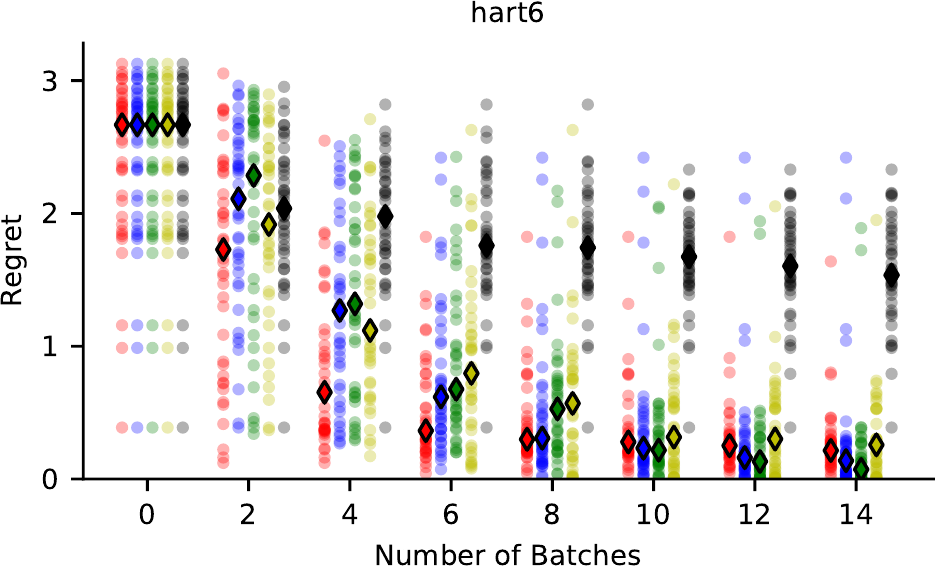}
	\end{subfigure}
	\caption{BO of \emph{batch size 5} on synthetic functions. Red, blue, green, yellow and black dots depict runs of OEI, LP-EI, BLCB, QEI and Random algorithms respectively. Diamonds depict the median regret for each algorithm.}
	\label{fig:SyntheticBatch5}
\end{figure}
\begin{figure}[!ht]
	\centering
	\begin{subfigure}{.49\textwidth}
		\centering
		\includegraphics[width=1\textwidth]{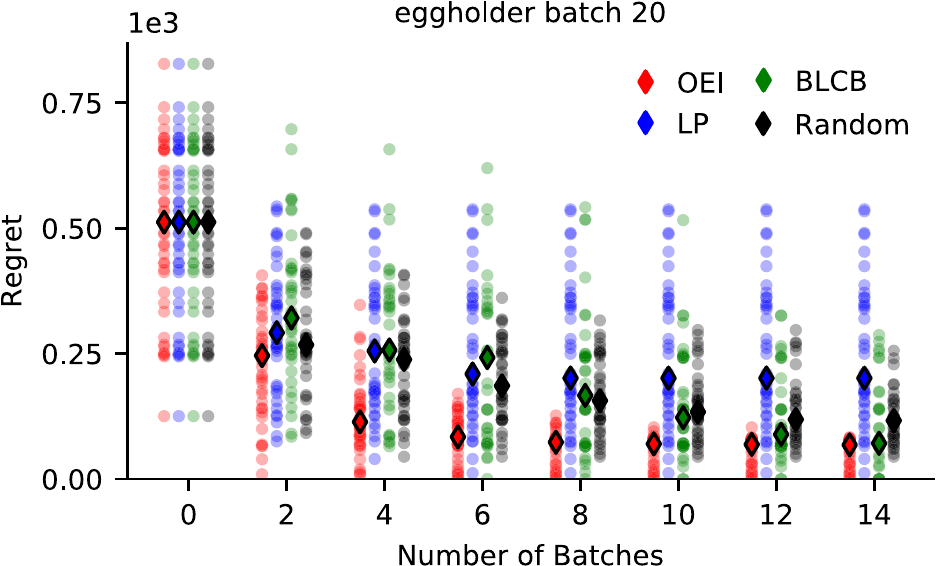}
	\end{subfigure}
	\begin{subfigure}{.49\textwidth}
		\centering
		\includegraphics[width=1\textwidth]{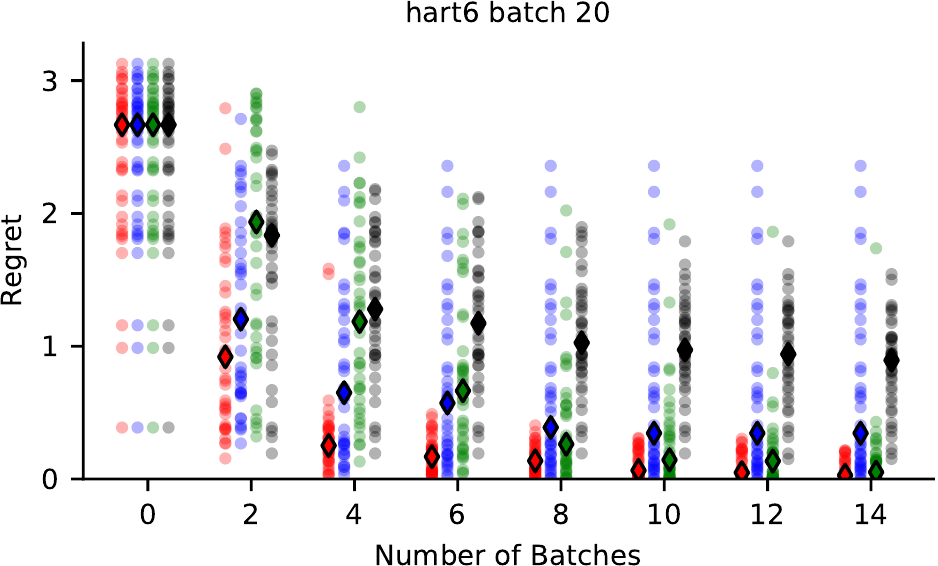}
	\end{subfigure}
	\caption{BO of \emph{batch size 20} on the challenging Hartmann 6-d Eggholder 2-d functions where OEI exhibits clearly superior performance. Red, blue, green and black dots depict runs of OEI, LP-EI, BLCB and Random algorithms respectively. Diamonds depict the median regret for each algorithm. Compare the above results with Figure \ref{fig:SyntheticBatch5} for the same runs with a smaller batch size. Note that QEI is not included in this case, as it does not scale to large batch sizes (see Table \ref{table:Timings}).}
	\label{fig:SyntheticBatch20}
\end{figure}
Next, we evaluate the performance of OEI in minimizing synthetic benchmark functions. The functions considered are: the Six-Hump Camel function defined on $[-2, 2] \times [-1, 1]$, the Hartmann 6 dimensional function defined on $[0, 1]^6$ and the Eggholder function, defined on $[-512, 512]^2$. We compare the performance of OEI against QEI, LP-EI and BLCB as well as random uniform sampling. The initial dataset consists of 10 random points for all the functions. A Matern 3/2 kernel is used for the GP modeling \citep{Rasmussen2005}. As all of the considered functions are noiseless, we set the likelihood variance to a fixed small number $10^{-6}$ for numerical reasons. For the purpose of generality, the input domain of every function is scaled to $[-0.5, 0.5]^n$ while the observation dataset $y_d$ is normalized at each iteration, such that $\mathrm{Var}[y^d] = 1$. The same transformations are applied to QEI, LP-EI and BLCB for reasons of consistency. All the acquisition functions except OEI are optimized with the quasi-Newton L-BFGS-B algorithm \citep{Fletcher1987} with 20 random restarts. We use point estimates for the kernel's hyperparameters obtained by maximizing the marginal likelihood via L-BFGS-B restarted on 20 random initial points.

First, we consider a small-scale scenario of \emph{batch size 5}. The results of 40 runs of Bayesian Optimization on a mixture of Cosines, the Six-Hump Camel, Eggholder, and 6-d Hartmann functions are depicted in Figure \ref{fig:SyntheticBatch5}. In these functions, OEI approaches the minimums faster than QEI and BLCB while considerably outperforming LP-EI, which exhibits outliers with bad behavior. The explorative nature of OEI can be observed in the optimization of the Hartmann function. OEI quickly reaches the vicinity of the minimum, but then decides not to refine the solution further but explore instead the rest of the 6-d space. Increasing the \emph{batch size} to \emph{20} for the challenging Eggholder and Hartmann functions shows a further advantage for OEI. Indeed, as we can observe in Figure \ref{fig:SyntheticBatch20}, OEI successfully exploits the increased batch size. BLCB also improves its performance though not to the extent of OEI. In contrast, LP-EI fails to manage the increased batch size. This is partially expected due to the heuristic based nature of LP-EI: the Lipschitz constant estimated by LP-EI is rarely suitable for all the 20 locations of the suggested batch. Even worse, LP-EI's performance is \emph{deteriorated} as compared to smaller batch sizes. LP-EI is plagued by numerical issues in the calculation of its gradient, and suggests multiple re-evaluations of the same points. This multiple re-samplings affects the GP modeling, resulting in an inferior overall BO performance.

\subsection{Real world example: Tuning a Reinforcement Learning Algorithm on various tasks}
\begin{figure}[!ht]
	\centering
	\begin{subfigure}{.49\textwidth}
		\centering
		\includegraphics[width=1\textwidth]{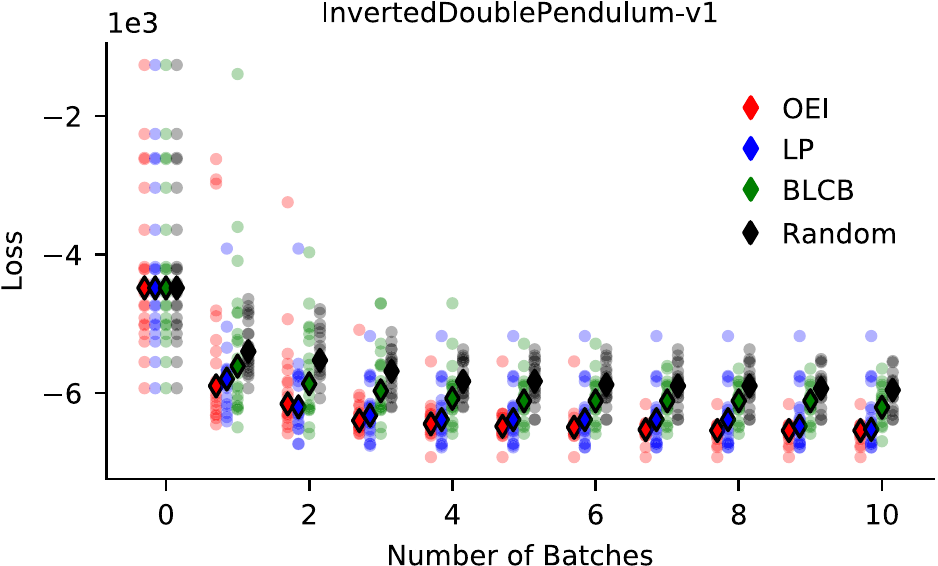}
	\end{subfigure}
	\centering
	\begin{subfigure}{.49\textwidth}
		\centering
		\includegraphics[width=1\textwidth]{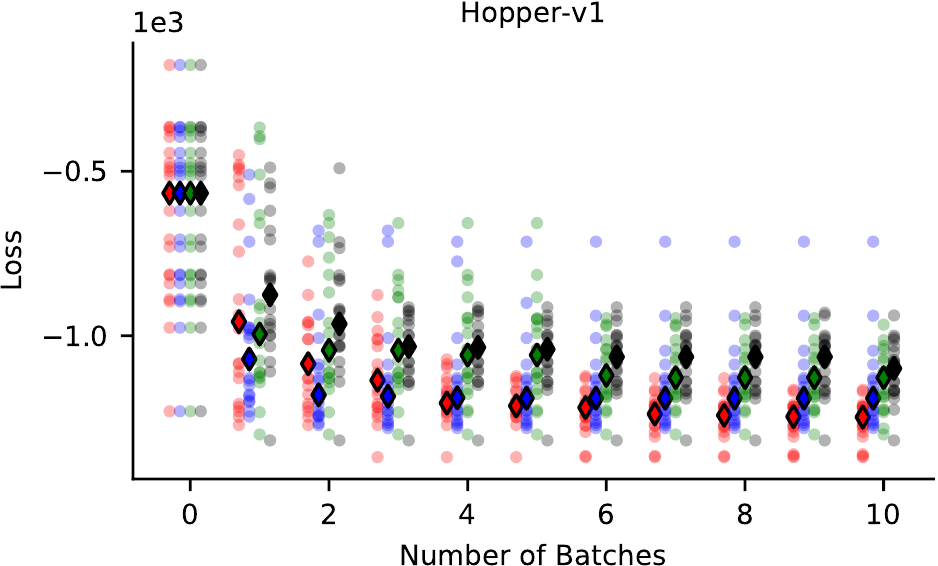}
	\end{subfigure}
	\centering
	\begin{subfigure}{.49\textwidth}
		\centering
		\includegraphics[width=1\textwidth]{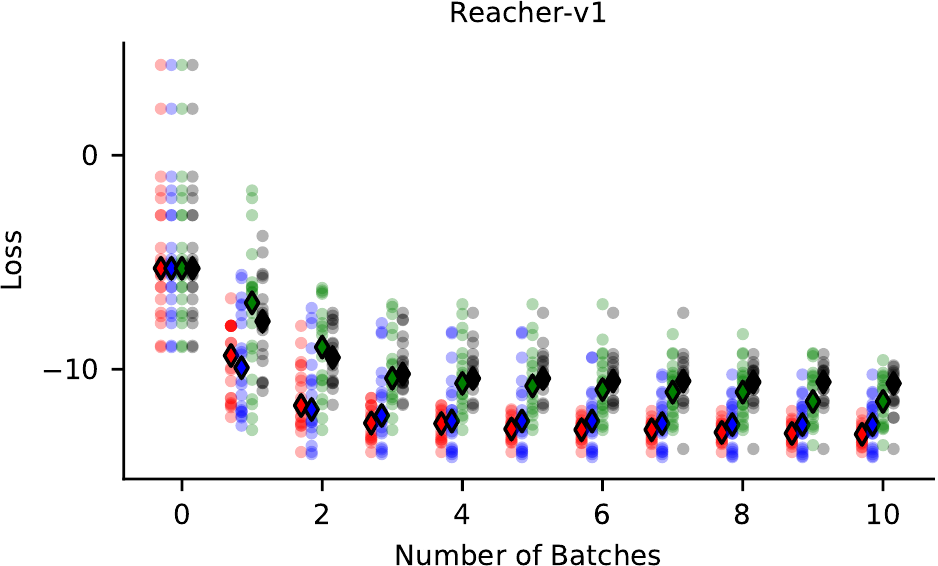}
	\end{subfigure}
	\centering
	\begin{subfigure}{.49\textwidth}
		\centering
		\includegraphics[width=1\textwidth]{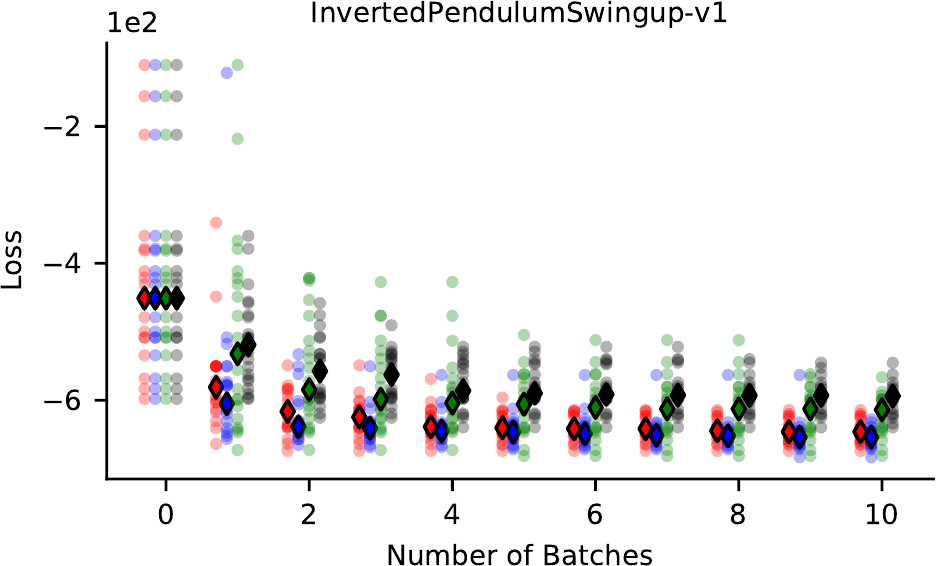}
	\end{subfigure}
	\caption{BO of batch size 20 on tuning PPO on a variety of robotic tasks. Red, blue, green and black dots depict runs of OEI, LP-EI, BLCB and Random algorithms respectively. Diamonds depict the median loss of all the runs for each algorithm.}
	\label{fig:PPOResults}
\end{figure}
Finally we perform Bayesian Optimization to tune Proximal Policy Optimization (PPO), a state-of-the-art Deep Reinforcement Learning algorithm that has been shown to outperform several policy gradient reinforcement learning algorithms \citep{Schulman2017}. The problem is particularly challenging, as deep reinforcement learning can be notoriously hard to tune, without any guarantees about convergence or performance. We use the implementation \citet{Dhariwal2017} published by the authors of PPO and tune the reinforcement algorithm on 4 \texttt{OpenAI~Gym} tasks (Hopper, InvertedDoublePendulum, Reacher and InvertedPendulumSwingup) using the \texttt{Roboschool} robot simulator. We tune a set of 5 hyper-parameters which are listed in Table \ref{table:hyperparameters}.  We define as objective function the negative average reward per episode over the entire training period ($4\cdot 10^5$ timesteps), which favors fast learning \citep{Schulman2017}. All of the other parameters are the same as the original implementation \citep{Schulman2017}.

We run Bayesian Optimization with batch size of 20, with the same modeling, preprocessing and optimization choices as the ones used in the benchmark functions. The results of 20 runs are depicted in Figure \ref{fig:PPOResults}. OEI outperforms BLCB, which performs comparably to Random search, and LP-EI, which exhibits high variance with occasional outliers stuck in inferior solutions.
\begin{table}[t]
	\caption{List of PPO's Hyperparameters used for tuning. Items with asterisk are tuned in the log-space.}
	\label{table:hyperparameters}
	\vskip 0.15in
	\begin{center}
		\begin{tabular}{cc}
			\toprule
			{Hyperparameter}           & {Range}                      \\
			\midrule
			Adam step-size             & $[10^{-5}, 10^{-3}]^*$       \\
			Clipping parameter         & $[0.05, 0.5]$                \\
			Batch size                 & $24, \dots, 256$             \\
			Discount Factor $(\gamma)$ & $1 - [10^{-3}, 10^{-3/2}]^*$ \\
			GAE parameter $(\lambda)$  & $1 - [10^{-2}, 10^{-1}]^*$   \\
			\bottomrule
		\end{tabular}
	\end{center}
	\vskip -0.1in
\end{table}

\section{Conclusions}
We have introduced a new acquisition function that is a tractable, probabilistic relaxation of the multi-point Expected Improvement, drawing ideas from the Distributionally Ambiguous Optimization community. Novel theoretical results allowed inexpensive calculation of first and second derivative information resulting in efficient Bayesian Optimization on large batch sizes.

\acks{This work was supported by the EPSRC AIMS CDT grant EP/L015987/1 and Schlumberger. The authors would like to acknowledge the use of the University of Oxford Advanced Research Computing (ARC) facility in carrying out this work. \url{http://dx.doi.org/10.5281/zenodo.22558}. Many thanks to Leonard Berrada for various useful discussions.}


\appendix

\section{Value of the Expected Improvement Lower Bound}
\label{app:PrimalOptProb}

In this section we provide a proof of the first of our main results, Theorem \ref{thm:ValueTheorem}, which establishes that for any $\Sigma \succ 0$ one can compute the value of our optimistic lower bound function
\begin{equation}\label{app:lowerBoundAbstract}
	\inf_{\mathbb P \in \mathcal{P}(\mu, \Sigma)}\mathbb{E}_{\mathbb{P}}[g(\xi)]
\end{equation}
via solution of a convex optimization problem in the form of a semidefinite program.

\subparagraph{{\it Proof of Theorem~\ref{thm:ValueTheorem}}:}~\\[1ex]
Recall that the set $\mathcal{P}(\mu, \Sigma)$ is the set of all distributions with mean $\mu$ and covariance $\Sigma$. Following the approach of \citet[Lemma 1]{Zymler2013}, we first remodel problem \eqref{app:lowerBoundAbstract} as:
\begin{equation}
	\begin{aligned}
		& {\inf_{\nu \in \mathcal{M}_+}}
		& & \int_{\mathbb{R}^k}
		{g(\xi)\nu(d\xi)}\\
		& \text{s.t.}
		& & \int_{\mathbb{R}^k}{\nu(d\xi)} = 1\\
		  &   &   & \int_{\mathbb{R}^k}{\xi \nu(d\xi)} = \mu \\
		  &   &   & \int_{\mathbb{R}^k}{\xi \xi^T
		\nu(d\xi)} = \Sigma + \mu
		\mu^T,\\
	\end{aligned}
	\label{eqn:PrimalWorstCase}
\end{equation}
where $\mathcal{M}_+$ represents the cone of nonnegative Borel measures on $\mathbb{R}^k$.  The optimization problem \eqref{eqn:PrimalWorstCase} is a semi-infinite linear program, with infinite dimensional decision variable $\nu$ and a finite collection of linear equalities in the form of moment constraints.

As shown by \citet{Zymler2013}, the dual of problem \eqref{eqn:PrimalWorstCase} has instead a finite dimensional set of decision variables and an infinite collection of constraints, and can be written as
\begin{equation}
	\begin{aligned}
		& {\text{sup}}
		  &   & \langle \Omega, M \rangle \\
		& \text{s.t.}
		  &   & \big[\xi^T \; 1 \big] M
		\big[\xi^T \; 1
		\big]^T \leq g(\xi) \quad \forall \xi \in
		\mathbb{R}^k,
	\end{aligned}
	\label{eqn:SemiInfiniteOptProb}
\end{equation}
with $M \in \mathbb{S}^{k+1}$ the decision variable and $\Omega\in \mathbb{S}^{k+1}$ the second order moment matrix of $\xi$ (see \eqref{eqn:Omega}). Strong duality holds between problems~\eqref{eqn:PrimalWorstCase} and~\eqref{eqn:SemiInfiniteOptProb} for any $\Omega \succ 0 \Leftrightarrow \Sigma \succ 0$, i.e.\ there is zero duality gap and their optimal values coincide.

The dual decision variables in \eqref{eqn:SemiInfiniteOptProb} form a matrix $M$ of Lagrange multipliers for the constraints in \eqref{eqn:PrimalWorstCase} that is block decomposable as
\begin{equation*}
	M =
	\left(
	\begin{array}{cc}
		M_{11}   & m_{12} \\
		m_{12}^T & m_{22} \\
	\end{array}
	\right),
\end{equation*}
where $M_{11} \in \mathbb{S}^k$ are multipliers for the second moment constraint, $m_{12} \in \mathbb{R}^k$ multipliers for the mean value constraint, and $m_{22}$ a scalar multiplier for the constraint that $\nu \in \mathcal M_+$ should integrate to $1$, i.e.\ that $\nu$ should be a probability measure.

For our particular problem, we have:
\begin{equation*}
	\begin{aligned}
		g(\xi) & = \min(\xi_{(1)}, \dots, \xi_{(k)}, \underline{y^d}) - \underline{y^d}     \\
		       & = \min(\xi_{(1)} - \underline{y^d}, \dots, \xi_{(k)} - \underline{y^d}, 0)
	\end{aligned}
\end{equation*}
as defined in \eqref{eqn:definition_g}, so that \eqref{eqn:SemiInfiniteOptProb} can be rewritten as
\begin{equation}
	\begin{aligned}
		& {\text{sup}}
		& & \langle \Omega, M \rangle\\
		& \text{s.t.}
		& & \big[\xi^T \; 1 \big] M
		\big[\xi^T \; 1
		\big]^T \leq \xi_{(i)} - \underline{y^d},\qquad \forall \xi \in
		\mathbb{R}^k \\
		  &   &   & \big[\xi^T \; 1 \big] M
		\big[\xi^T \; 1
		\big]^T \leq 0
		\quad \quad i = 1,\dots,k
	\end{aligned}
	\label{eqn:SemiInfiniteDualOptProb}
\end{equation}

The infinite dimensional constraints in \eqref{eqn:SemiInfiniteDualOptProb} can be replaced by the equivalent conic constraints
\begin{equation*}
	M -
	\begin{bmatrix}
		0       & e_i/2            \\
		e_i^T/2 & -\underline{y^d}
	\end{bmatrix}
	\preceq 0,
	\qquad i = 1,\dots,k
\end{equation*}
and $M \preceq 0$ respectively, where $e_i$ are the standard basis vectors in $\mathbb{R}^k$. Substituting the above constraints in \eqref{eqn:SemiInfiniteDualOptProb} results in \eqref{eqn:PrimalOptProb}, which proves the claim. \hfill\BlackBox

\section{Gradient of the Expected Improvement Lower Bound} \label{app:DerivativeCalculations}

In this section we provide a proof of our second main result, Theorem \ref{thm:gradientTheorem}, which shows that the gradient\footnote{Technically, the gradient is not defined, as $\Omega$ is by definition symmetric. We can get around this technicality by a slight abuse of notation allowing for a non-symmetric $\Omega$ such that $\Omega + \Omega^T \in \mathbb{S}^{k+1}_{++}$.} $\partial p / \partial \Omega$ of our lower bound function \eqref{eqn:ExpectedImprovementBounds} with respect to $\Omega$ coincides with the optimal solution of the semidefinite program \eqref{eqn:PrimalOptProb}.  We will find it useful to exploit also the dual of this SDP, which we can write as
\begin{equation}
	\begin{aligned}
		\quad & {\text{inf}}
		& & \sum_{i=1}^{k} \langle Y_i, C_i \rangle \\
		& \text{s.t.}
		& & Y_i \succeq 0, \quad i = 0,\dots,k \\
			&   &   & \sum_{i=0}^{k}Y_i = \Omega.
	\end{aligned}
	\label{eqn:DualOptProb}
\end{equation}
and the Karush-Kuhn-Tucker conditions for the pair of primal and dual solutions $\bar M$, $\{\bar Y_i\}$:
\begin{align}
	\label{eqn:KKT1}
	C_i - \bar M                         & \succeq 0 \\
	\label{eqn:KKT2}
	\bar Y_i                             & \succeq 0 \\
	\label{eqn:KKT3}
	\langle \bar Y_i, \bar M - C_i \rangle= 0
	\Rightarrow  \bar Y_i (\bar M - C_i) & = 0       \\
	\label{eqn:KKT4}
	\frac
	{\partial \mathcal{L}(M,\Omega)}
	{\partial M} \Bigg|_{\bar M} = 0
	\Rightarrow
	\sum_{i=0}^{k}\bar Y_i               & = \Omega,
\end{align}
where $\mathcal{L}$ denotes the Lagrangian of \eqref{eqn:PrimalOptProb}.

Before proving Theorem \ref{thm:gradientTheorem} we require three ancillary results. The first of these results establishes that any feasible point $M$ for the optimization problem \eqref{eqn:PrimalOptProb} has strictly negative definite principal minors in the upper left hand corner.
\begin{lemma}
	For any feasible $M \in \mathbb{S}^{k+1}$ of \eqref{eqn:PrimalOptProb} the upper left $k \times k$ matrix $M_{11}$ is negative definite.
	\label{lem:PositiveDefinite}
\end{lemma}
\begin{proof}
	Let
	\[
		M =
		\begin{bmatrix}
			M_{11}   & m_{12} \\
			m_{12}^T & m_{22}
		\end{bmatrix}
	\]
	where $M_{11} \in \mathbb{S}^{k}, m_{12} \in \mathbb{R}^{k}$ and $m_{22} \in \mathbb{R}$.

	From \eqref{eqn:PrimalConstraints} we can infer that $m_{22} + \underline{y^d} < 0$, otherwise \eqref{eqn:PrimalConstraints} would require $m_{12} + e_i/2 = 0 \; \forall i = 1, \dots k$, which is impossible.

	Since $M \preceq 0$, we have $m_{22} \leq 0$. Assume though, for now, that $m_{22} < 0$. Applying then a standard Schur complement identity in \eqref{eqn:PrimalConstraints} results in:
	\begin{equation*}
		\begin{aligned}
			M_{11} & \preceq (m_{12} - e_i) (m_{22} + \underline{y^d})^{-1} (m_{12} - e_i)^T
			\\
			M_{11} & \preceq m_{12} m_{22}^{-1} m_{12}^T
			\quad i = 1, \dots, k
		\end{aligned}
	\end{equation*}
	Summing the above results in
	\begin{equation*}
		M_{11} \preceq \frac{(m_{22} + \underline{y^d})^{-1}}{k + 1} \sum_{i=1}^{k}(m_{12} - e_i) (m_{12} - e_i)^T
		+ \frac{m_{22}^{-1}}{k + 1} m_{12}m_{12}^T,
		\label{eqn:SchurResult}
	\end{equation*}
	which results in $M_{11} \prec 0$, since
	\text{span}$\big(\{m_{12}, \{m_{12} - e_i\}_{i=i,\dots,k}\}\big) \supseteq$
	\text{span}$\big(\{m_{12} - m_{12} + e_i\}_{i=i,\dots,k}\}\big)$
	$= \mathbb{R}^k$.

	Finally, in the case where $m_{22} = 0$ we have $m_{12} = 0$, since $M \preceq 0$. Applying the above results in $M_{11} \preceq (m_{22}+\underline{y^d})^{-1}/k \sum_{i=1}^{k}\frac{1}{k}e_i e_i^T \prec 0.$
\end{proof}

The second auxiliary results lists some useful properties of the dual solution:
\begin{lemma} \label{lem:dualProperties}
	The optimal Lagrange multipliers of \eqref{eqn:PrimalOptProb} are of rank one with 
	$\mathcal{R}(\bar Y_i) = \mathcal{N}(\bar M - C_i), \; \forall i = 0,\dots,k$, where $\mathcal{N}(\cdot)$ and $\mathcal{R}(\cdot)$ denote the nullspace and the range of a matrix.
\end{lemma}
\begin{proof}
	Lemma~\ref{lem:PositiveDefinite} implies that $[x^T 0](\bar M-C_i)[x^T 0]^T = [x^T 0]\bar M[x^T 0]^T < 0, \forall x\in \mathbb{R}^k$ (recall that $C_i$ is nonzero only in the last column or the last row), which means that $\rank(\bar M-C_i) \geq k$.  Due to the complementary slackness condition \eqref{eqn:KKT3}, the span of $\bar Y_i$ is orthogonal to the span of $\bar M-C_i$ and consequently $\rank(\bar Y_i) \leq 1$. However, according to \eqref{eqn:KKT4} we have
	\[
		\rank \sum_{i=0}^{k} \bar Y_i = \rank(\Omega)
		\stackrel{\Omega \succ 0}{\Longrightarrow}
		\sum_{i = 0}^{k} \rank(\bar Y_i) \geq k+1
	\]
	which results in
	\[
		\rank(\bar M - C_i) = k, \quad \rank(\bar Y_i) = 1,
	\]
	and, using \eqref{eqn:KKT3}:
	\[
		\mathcal{R}(\bar Y_i) = \mathcal{N}(\bar M-C_i),
		\quad
		i = 0,\dots,k.
	\]
\end{proof}

Our final ancillary result considers the (directional) derivative of the function $p$ when its argument is varied linearly along some direction $\bar \Omega$.
\begin{lemma}
	Given any $\bar\Omega \in \mathbb{S}^{k+1}$ and any moment matrix $\Omega \in \mathbb{S}_{++}^{k+1}$, define the  scalar function $q(\cdot\,; \Omega) : \mathbb{R} \to \mathbb R$ as
	\[
		q(\gamma; \Omega) \eqdef p(\Omega + \gamma \bar \Omega).
	\]
	Then $q(\cdot\,;\Omega)$ is differentiable at $0$ with $\partial q(\gamma; \Omega)/ \partial \gamma |_{\gamma = 0} = \langle \bar \Omega, \bar M(\Omega) \rangle$, where $\bar M(\Omega)$ is the optimal solution of \eqref{eqn:PrimalOptProb} at $\Omega$.
	\label{lem:appSupportLemma1}
\end{lemma}
\begin{proof}
	Define the set $\Gamma_\Omega$ as
	\[
		\Gamma_\Omega := \set{\gamma}{\gamma\in\domain q(\cdot\,;\Omega)}
		= \set{\gamma}{\left(\Omega + \gamma\bar\Omega\right) \in \domain p},
	\]
	i.e.\ the set of all $\gamma$ for which problem \eqref{eqn:PrimalOptProb} has a bounded solution given the moment matrix $\Omega + \gamma\bar\Omega$.
	In order to prove the result it is then sufficient to show:
	\begin{enumerate}[label=\roman*]
		\item $0 \in \interior \Gamma_\Omega$, and
		\item The solution of \eqref{eqn:PrimalOptProb} at $\Omega$ is
		      unique.
	\end{enumerate}

	The remainder of the proof then follows from \citet[Lemma 3.3]{Goldfarb1999}, wherein it is shown that $0\in \interior \Gamma_\Omega$ implies that $\langle \bar \Omega, \bar M(\Omega) \rangle$ is a subgradient of $q(\cdot\,; \Omega)$ at $0$, and subsequent remarks in \citet{Goldfarb1999} establish that uniqueness of the solution $M(\Omega)$ ensure differentiability.

	We will now show that both of the conditions (i) and (ii) above are satisfied.
	\subparagraph{{\it (i):
		Proof that $0 \in \interior \Gamma_\Omega$:}}~\\[1ex]
	It is well-known that if both of the primal and dual problems \eqref{eqn:PrimalOptProb} and \eqref{eqn:DualOptProb} are strictly feasible then their optimal values coincide, i.e.\ Slater's condition holds and we obtain strong duality; see \citep[Section 5.2.3]{Boyd2004} and \citep{Ramana1997}.

	For \eqref{eqn:PrimalOptProb} it is obvious that one can construct a strictly feasible point. For \eqref{eqn:DualOptProb}, $Y_i = \Omega/(k+1)$ defines a strictly feasible point for any $\Omega \succ 0$. Hence \eqref{eqn:PrimalOptProb} is solvable for any $\Omega + \gamma \bar \Omega$ with $\gamma$ sufficiently small. As a result, $0 \in \interior \Gamma$.

	\subparagraph{{\it (ii): Proof that the solution to
		\eqref{eqn:PrimalOptProb} at $\Omega$ is unique:}}~\\[1ex]%
	We will begin by showing that the range of the dual variables $\mathcal{R}(Y_i), i = 0, \dots k$ remain the same for every primal-dual solution. Assume that there exists another optimal primal-dual pair denoted by $\tilde M = \bar M + \delta M$, and ${\tilde Y_i}$.
	Due to Lemma \ref{lem:dualProperties}, there exist $\bar y_i, \tilde y_i \in \mathbb{R}^{k+1}$ such that
	\begin{equation}
		\bar Y_i = \bar y_i \bar y_i^T, \quad \tilde Y_i = \tilde y_i \tilde y_i^T  \quad \forall i = 0,\dots,k.
		\label{eqn:YDecomposition}
	\end{equation}
	Obviously, $\tilde y_i \in \mathcal{R}(\tilde Y_i) = \mathcal{N}(\tilde M - C_i)$ and, by definition, we have
	\begin{equation}
		\label{eqn:deltaMNullVectors}
		\tilde y_i^T (\tilde M - C_i) \tilde y_i =
		0 \quad i = 0,\dots,k.
	\end{equation}
	Moreover, as $\bar M$ is feasible we have $\tilde y_i^T(\bar M-C_i) \tilde y_i \leq 0$, resulting in
	\begin{equation}
		\label{eqn:deltaMNegativity}
		\tilde y_i^T \delta M \tilde y_i \geq 0,
		\quad i = 0,\dots,k.
	\end{equation}
	Since $\bar M$ and $\tilde M$ have the same objective value we conclude that $\langle \Omega, \delta M \rangle = 0$. Moreover, according to \eqref{eqn:KKT4} and \eqref{eqn:YDecomposition} we can decompose $\Omega$ as $\sum _{i=0} ^k \tilde y_i \tilde y_i^T$. Hence
	\begin{equation*}
		\begin{aligned}
			\tr(\Omega\delta M) = 0
			& \Longrightarrow
			\quad \tr(\delta M \sum_{i=0}^{k} \tilde y_i \tilde y_i^T) = 0
			\Longrightarrow
			\quad \sum_{i=0}^{k} \tr (\delta M \tilde y_i \tilde y_i^T)= 0 \\
			& \Longrightarrow
			\quad \sum_{i=0}^{k} \tilde y_i^T \delta M \tilde y_i = 0
			\stackrel{\eqref{eqn:deltaMNegativity}}{\Longrightarrow}\quad
			\tilde y_i^T \delta M \tilde y_i = 0
			& \quad \forall i = 0,\dots,k \\
			& \stackrel{\eqref{eqn:deltaMNullVectors}}{\Longrightarrow}\quad
			\tilde y_i (\bar M - C_i) \tilde y_i^T = 0
			& \quad \forall i = 0,\dots,k.
		\end{aligned}
	\end{equation*}
	Hence, $\tilde y_i$ is, like $\bar y_i$, a null vector of $\bar M - C_i$. Since the null space of $\bar M - C_i$ is of rank one, we get $\tilde y_i = \lambda_i \bar y_i$ for some $\lambda_i \in \mathbb{R}$, resulting in, $\bar Y_i = \lambda_i^2 \tilde Y_i$.

	Now we can show that the dual solution is unique. Assume that $\bar Y_i \neq \tilde Y_i$, i.e. $\lambda_i^2 \neq 1$ for some $i = \{1, \dots, k\}$. Feasibility of $\tilde Y_i$ and $\bar Y_i$ gives
	\begin{equation*}
		\begin{array}{l@{}|}
			\begin{aligned}
			\sum _{i=0} ^k \bar Y_i = \Omega \Leftrightarrow \sum _{i=0} ^k \lambda_i^2 \tilde Y_i & = \Omega \quad \\
			\sum _{i=0} ^k \tilde Y_i                                                              & = \Omega
			\end{aligned}
		\end{array}
		\Rightarrow
		\sum _{i=0} ^k (1 - \lambda_i^2) \tilde Y_i = 0
	\end{equation*}
	i.e. $\{\mathcal{R}(\tilde Y_i)\}$ are linearly dependent. This contradicts Lemma \ref{lem:dualProperties} and \eqref{eqn:KKT4} which suggest linear independence, as each $\mathcal{R}(Y_i)$ is of rank one and together they span the whole space $\mathbb{R}^{k+1}$. Hence, $\bar Y_i = \tilde Y_i \; \forall i = \{0, \dots, k \}$, i.e. the dual solution is unique.

	Finally, the uniqueness of the primal solution can be established by the uniqueness for the dual solution. Indeed, summing \eqref{eqn:KKT3} gives
	\begin{equation*}
		\sum_{i=0}^{k} \bar Y_i (\bar M-C_i) = 0
		\stackrel{\eqref{eqn:KKT4}}{\Leftrightarrow} \Omega \bar M =
		\sum_{i=0}^{k} \bar Y_i C_i
		\Leftrightarrow \bar M = \Omega^{-1} \sum_{i=0}^{k} \bar
		Y_i C_i.
	\end{equation*}
\end{proof}

\subparagraph{{\it Proof of Theorem \ref{thm:gradientTheorem}}:}~\\[1ex]
Given the preceding support results of this section, we are now in a position to prove Theorem \ref{thm:gradientTheorem}.

First, we will show that $p: \mathbb{S}_{++}^{k+1} \mapsto \mathbb{R}$ is differentiable on its domain. First, note that $p$ is convex due to \citep[Corollary 3.32]{Rockafellar2009} and hence continuous on $\mathrm{int}(\mathrm{dom}\; p) = S_{++}^{k+1}$ \citep[Theorem 2.35]{Rockafellar2009}. Also, note that due to Lemma \ref{lem:appSupportLemma1} the regular directional derivatives \citep[Theorem 8.22]{Rockafellar2009} of $p$ exist and are a linear map of the direction.  Hence, according to \citep[Theorem 9.18 (a, f)]{Rockafellar2009} p is differentiable on $\mathbb{S}_{++}^{k+1}$.

Consider now the derivative of the solution of \eqref{eqn:PrimalOptProb} when perturbing $\Omega$ across a specific direction $\bar \Omega $, i.e. $\partial q(\gamma; \Omega) / \partial \gamma$ with $q(\gamma; \Omega) = p(\Omega + \gamma \bar \Omega)$. Lemma \ref{lem:appSupportLemma1} shows that $\partial q(\gamma; \Omega) / \partial \gamma|_{0} = \langle \bar \Omega, \bar M \rangle$ when $\Omega \succ 0$. The proof then follows element-wise from Lemma \ref{lem:appSupportLemma1} by choosing $\bar \Omega$ a sparse symmetric matrix with $\bar \Omega_{(i,j)} = \bar \Omega_{(j,i)} = 1/2$ the only nonzero elements. \hfill\BlackBox

\section{Derivative of the Optimal Solution}
\label{app:Mdot}
In this section we will provide a constructive proof of Theorem \ref{thm:dotM}, and show in particular that $\dot{\bar M}$, the directional derivative of $\bar M(\Omega)$ when perturbing $\Omega$ linearly across a direction $\bar \Omega \in \mathbb{S}^{k+1}$, can be computed by solution of the following linear system
\begin{equation*}
	\begin{aligned}
		\begin{bmatrix}
		\bar S_0                      &        & 0                             & (\bar y_0^T \otimes I) \Pi_+ \\
																	& \ddots &                               & \vdots                       \\
		0                             &        & \bar S_k                      & (\bar y_k^T \otimes I) \Pi_+ \\
		\Pi(\bar y_0 \oplus \bar y_0) & \dots  & \Pi(\bar y_0 \oplus \bar y_0) & 0                            \\
		\end{bmatrix}
		\begin{bmatrix}
		\dot{\bar y}_0 \\
		\vdots \\
		\dot{\bar y}_k \\
		\vcu(\dot{\bar M})
		\end{bmatrix}
		=
		\begin{bmatrix}
		0 \\
		\vdots \\
		0 \\
		\vcu(\bar \Omega)
		\end{bmatrix}
	\end{aligned}
\end{equation*}
where
\begin{itemize}
	\item $\bar S_i = \bar M_i - C_i, \quad i = 0, \dots,  k$
	\item $y_i$ is defined such that $\bar Y_i = \bar y_i \bar y_i^T$, i.e. the non-zero eigenvector of the Lagrange multiplier $\bar Y_i$, which was shown to be unique in Lemma \eqref{lem:dualProperties}.
	\item $\vc(\cdot)$, is the operator that stacks the columns of a matrix to a vector, and $\vcu(\cdot)$ the operator that stacks only the upper triangular elements in a similar fashion
	\item $\Pi$, is the matrix that maps $\vc(Z) \mapsto \vcu(Z)$ where $Z\in \mathbb{S}$, and $\Pi^+$ performing the inverse operation.
	\item $\otimes$ and $\oplus$ denote the Kronecker product and sum respectively.
\end{itemize}
\begin{proof}
	Lemma \ref{lem:appSupportLemma1} of Appendix \ref{app:DerivativeCalculations}, guarantees that the solutions of \eqref{eqn:PrimalOptProb} and \eqref{eqn:DualOptProb} are unique for any $\Omega \succ 0$. Hence, according to \cite{Freund2004}, the directional derivatives $\dot{\bar M}, \dot{\bar Y}_i$ of $\bar M, \bar Y$ along the perturbation $\bar \Omega$ exist and are given as the unique solution to the following overdetermined system:
	\begin{equation}
		\begin{aligned}
			\sum_{i=0}^{k} \dot{\bar Y}_i                   & = \bar \Omega                               \\
			\dot{\bar Y}_i \bar S_i - \bar Y_i \dot{\bar M} & = 0                                         \\
			\dot{\bar M}, \dot{\bar Y}_i                    & \in \mathbb{S}^{k+1} \quad i = 0, \dots, k,
			\label{eqn:derivativesFreund}
		\end{aligned}
	\end{equation}
	The above linear system is over-determined, and has symmetric constraints. This renders standard solution methods, such as LU decomposition, inapplicable. Expressing the above system in a standard matrix form results in a matrix with $\mathcal{O}(k^4)$ zeros, which makes its solution very costly.

	To avoid theses issues, we will exploit Lemma \ref{lem:dualProperties} of Appendix \ref{app:DerivativeCalculations} to express the dual solution $\bar Y_i$ compactly as $\bar Y_i = \bar y_i \bar y_i^T$. One can choose a differentiable mapping $\bar Y_i(t) \mapsto \bar y_i(t)$, e.g. $\bar y_i(t) = \sqrt{\lambda_i(t)} u_i(t)$ where $\lambda_i(t)$ is the only positive eigenvalue of $\bar Y_i(t)$ and $u_i(t)$ its corresponding unit-norm eigenvector. Differentiability of $\bar y_i(t)$ comes from differentiability of $\bar Y_i(t)$, $\lambda_i(t)$, $u_i(t)$ \citep{Kato1976} and positivity of $\lambda_i(t)$ due to Lemma \ref{lem:dualProperties}. The chain rule then applies for $\dot{\bar Y}_i = \dot{\bar y}_i \bar y_i^T + \bar y_i \dot{\bar y}_i^T$.
	Hence \eqref{eqn:derivativesFreund} can be expressed as
	\begin{align}
		\sum_{i=0}^{k}
		\dot{\bar y}_i \bar y_i^T + \bar y_i \dot{\bar y}_i^T
		  & = \bar \Omega
		\label{eqn:sumY} \\
		(\dot{\bar y}_i \bar y_i^T + \bar y_i \dot{\bar y}_i^T) \bar S_i
		- \bar y_i \bar y_i^T \dot{\bar M}
		  & = 0, \quad i = 0, \dots, k
		\label{eqn:bigY}
	\end{align}
	Exploiting $\bar y_i^T \bar S_i = 0$ from \eqref{eqn:KKT3} and that $y_i \neq 0$ gives
	\begin{equation}
		\eqref{eqn:bigY} \Leftrightarrow \bar y_i(\dot{\bar y}_i^T\bar S_i - \bar y_i^T \dot{\bar M}) = 0
		\stackrel{y_i \neq 0}{\Longleftrightarrow}
		\dot{\bar y}_i^T\bar S_i - \bar y_i^T \dot{\bar M}
		, \quad i = 0, \dots, k
		\label{eqn:smallY}
	\end{equation}

	We can express equations \eqref{eqn:sumY} and \eqref{eqn:smallY} into the standard matrix form by using the $\mathrm{vec}$ operator and the identity $\vc(AXB) = (B^T \otimes A) \vc(X)$, which gives
	\begin{align*}
		\sum_{i=0}^{k}
		(\bar y_i \otimes I + I \otimes \bar y_i) \dot{\bar y}_i
		  & = \mathrm{vec}(\bar \Omega) \\
		S_i \dot{\bar y}_i - (\bar y_i^T \otimes I) \mathrm{vec}(\dot{\bar M})
		  & = 0, \quad i = 0, \dots, k
	\end{align*}
	Finally, eliminating the symmetric constraint via $\vcu(\cdot)$, $\Pi$ and $\Pi^+$ gives:
	\begin{equation}
		\label{eqn:directionalDerivativesM}
		\begin{aligned}
			\sum_{i=0}^{k}
			\Pi(\bar y_i \otimes I + I \otimes \bar y_i) \dot{\bar y}_i
			  & = \mathrm{vec_u}(\bar \Omega)
			\\
			S_i \dot{\bar y}_i - (\bar y_i^T \otimes I) \Pi^+\mathrm{vec_u}(\dot{\bar M})
			  & = 0, \quad i = 0, \dots, k
		\end{aligned}
	\end{equation}
	leading to the suggested linear system. The system is square and non-singular since it is equivalent to \eqref{eqn:derivativesFreund} which has a unique solution.

	Finally, it remains to show that $M(\Omega)$ is a differentiable for any $\Omega \succ 0$. First note that $M$ is outer semicontinuous \citep[Definition 5.4]{Rockafellar2009} on $\mathbb{S}_{++}^{k+1}$ as, according to Theorem \eqref{thm:gradientTheorem}, $p$ is continuous on $\mathbb{S}_{++}^{k+1}$. Since $M$ is unique on $\mathbb{S}_{++}^{k+1}$ it is also continuous. Finally, note that due to Equation \eqref{eqn:directionalDerivativesM} the regular directional derivatives \citep[Theorem 8.22]{Rockafellar2009} of $M$ exist and are a linear map of the direction.  Hence, according to \citep[Theorem 9.18 (a, f)]{Rockafellar2009} M is differentiable on $\mathbb{S}_{++}^{k+1}$.
\end{proof}

\section{Non-differentiability of QEI-OEI} \label{app:nonDifferentiability}
\begin{proposition}
	\label{prop:nonDifferentiability}
	The multipoint expected improvement
	\begin{equation*}
		\begin{gathered}
			\alpha(X) = \mathbb{E}[\min(y_{1}(X), \dots, y_{k}(X), \underline{y^d}) |
			\mathcal{D}] - \underline{y^d} \\
			\quad \text{with} \; y | \mathcal{D} \sim \mathcal{N}\bigl(\mu(X),
			\Sigma(X)\bigr)
		\end{gathered}
	\end{equation*}
	is non-differentiable for a smooth, ``non-trivial''\footnote{A kernel that does not induce a constant acquisition function.} kernel function that implies perfect correlation between $y_1(x_1)$ and $y_2(x_2)$ when $x_1 = x_2$.
\end{proposition}
\begin{proof}
	For simplicity, we will consider the two-points expected improvement
	\begin{equation*}
		\begin{gathered}
			\alpha_2(x_1, x_2) = \mathbb{E}[\min(y_{1}, y_{2}, \underline{y^d}) |
			\mathcal{D}] - \underline{y^d} \\
			\quad \text{with} \; y | \mathcal{D} \sim \mathcal{N}\bigl(\mu([x_1, x_2]),
			\Sigma([x_1, x_2])\bigr)
		\end{gathered}
	\end{equation*}
	We will show that the directional derivatives of $\alpha_2$ can be inconsistent, i.e. not a linear map of the direction.

	Assume that $a_2$ is differentiable and consider an $x$ for which the one-point expected improvement
	\begin{equation*}
		\alpha_1(x_1) = \mathbb{E}[\min(y_1, \underline{y^d}) |
		\mathcal{D}] - \underline{y^d} \\
		\quad \text{with} \; y | \mathcal{D} \sim \mathcal{N}\bigl(\mu([x_1]),
		\Sigma([x_1])\bigr).
	\end{equation*}
	is differentiable and non-stationary, i.e. there exists $\bar x$ such that $d\alpha_1(x + t \bar x)/dt < 0$. We can verify that such a point always exists if the kernel is smooth and non-trivial by examining \cite[Equation (6.2.2)]{Osborne2010}. Due to symmetry of $\alpha_2$ we have
	\begin{equation*}
		\frac{d\alpha_2(x + t \bar x, x)}{dt} =
		\frac{d\alpha_2(x, x + t \bar x)}{dt}
	\end{equation*}
	and as $\alpha_2(x, y) < \alpha_1(x) \leq 0$ and $\alpha_2(x, x) = \alpha(x) \; \forall x, y \in \mathbb{R}^n$, we get
	\begin{equation*}
		\frac{d\alpha_2(x + t \bar x, x)}{dt} \Bigg |_{t=0} \leq
		\frac{d\alpha(x + t \bar x)}{dt} \Bigg |_{t=0}  < 0
	\end{equation*}
	and
	\begin{equation*}
		\frac{d\alpha_2(x + t, x - t \bar x)}{dt} \Bigg |_{t=0} \leq
		\frac{d\alpha(x + t \bar x)}{dt} \Bigg |_{t=0} < 0
	\end{equation*}
	Hence we arrive at a contradiction, as the directional derivatives of $\alpha_2$ are not a linear map of the direction
	\begin{equation*}
		\frac{d\alpha_2(x + t, x - t \bar x)}{dt} \Bigg |_{t=0} \neq
		\frac{d\alpha_2(x + t \bar x, x)}{dt} \Bigg |_{t=0} -
		\frac{d\alpha_2(x, x - t \bar x)}{dt} \Bigg |_{t=0} = 0
	\end{equation*}
	It is trivial to see that the same argument holds for OEI.
\end{proof}

\vskip 0.2in
\bibliography{../../Bibliography/bibliography}

\end{document}